\documentclass{article}

\PassOptionsToPackage{numbers, compress}{natbib}
\usepackage[main, final]{neurips_2026}

\usepackage[utf8]{inputenc} 
\usepackage[T1]{fontenc}    
\usepackage{hyperref}       
\usepackage{url}            
\usepackage{booktabs}       
\usepackage{amsfonts}       
\usepackage{nicefrac}       
\usepackage{microtype}      
\usepackage{xcolor}         

\usepackage{amsmath}
\usepackage{amsthm}
\usepackage{graphicx}

\newtheorem{theorem}{Theorem}[section]

\newtheorem{lemma}[theorem]{Lemma}

\theoremstyle{definition}
\newtheorem{definition}[theorem]{Definition}
\theoremstyle{remark}
\newtheorem*{remark}{Remark}

\usepackage[linesnumbered,ruled]{algorithm2e}
\SetKwInput{KwInput}{Input}
\SetKwInput{KwOutput}{Output}
\SetKwInput{KwGlobal}{Global Parameter}
\newcommand\ber{\mathsf{Ber}}
\newcommand\abovethreshold{\mathsf{AboveThreshold}}
\newcommand\comp{\mathsf{Comp}}
\newcommand\alg{\mathsf{Alg}}
\newcommand\lap{\mathrm{Lap}}
\newcommand\A{\mathcal{A}}
\newcommand\D{\mathcal{D}}

\title{Private Online Prediction from Experts with Small Losses}

\author{%
  Bo Li$^{1, 2}$\quad Peng Ye$^{2}$\\
  $^{1}$Guangzhou HKUST Fok Ying Tung Research Institute \\
  $^{2}$Department of Computer Science and Engineering, HKUST \\
  \texttt{bli@ust.hk},~ \texttt{pyeac@connect.ust.hk}\\
}

\begin{document}

\maketitle

\begin{abstract}
  We consider online prediction from experts, a fundamental problem in machine learning, under differential privacy constraints. Existing private algorithms achieve near-optimal regret $\tilde{O}(\sqrt{T})$, where $T$ is the time horizon. However, this bound becomes suboptimal when $L^\star$, the cumulative loss of the best expert, is significantly less than $T$. In this work, we present the first differentially private algorithm with regret $\tilde{O}(\sqrt{L^\star})$, offering a substantial improvement in the low-loss regime. Our regret bound matches the non-private lower bound up to poly-logarithmic factors for constant privacy parameters, demonstrating that privacy incurs only a small cost in this setting. Our result assumes that the input sequence is oblivious, i.e., it is fixed in advance and independent of the algorithm's outputs. We also prove that an $\tilde{\Omega}(\sqrt{T})$ regret is inevitable for adaptively generated inputs, even when $L^\star = 0$.
\end{abstract}

\section{Introduction}

Online prediction from experts (OPE) is a fundamental problem in machine learning. In this problem, the learner receives a stream of $T$ loss functions $\ell_1,\dots,\ell_T$. In each round $t$, the learner chooses an expert $i_t\in[N]$ before seeing $\ell_t$ and suffers loss $\ell_t(i_t)$. The learner aims to minimize its total loss, which is compared to that of the best expert in hindsight, denoted by $L^\star$. The difference between these two quantities is called the regret, which measures the learner's performance for this problem.

A critical concern when deploying machine learning algorithms on sensitive data (e.g., financial or medical records) is privacy leakage. To address this, differential privacy (DP)~\citep{dwork2006calibrating,dwork2006our} has emerged as the de facto standard for formalizing the privacy guarantee provided by an algorithm. Informally, an algorithm is said to be differentially private if changing a single input data point alters its output distribution only negligibly. As a result, such a change in an individual's information is unlikely to be detected, thereby protecting personal privacy.

Given the importance of OPE, studying it under DP has received significant attention. Previous work~\citep{asi2023private,asi2024private} has developed algorithms that achieve regret $\tilde{O}(\sqrt{T})$, which is nearly optimal because it matches the $\Omega(\sqrt{T\log N}+\log N)$ lower bound~\citep{cesa1997use}---a bound that holds even without any privacy constraints. These results suggest that the cost of privacy in OPE is negligible in the worst case.

In many practical settings---such as financial forecasting, weather prediction, or personalized recommendations---domain experts are specifically selected for their accuracy, so their cumulative loss tends to be low over time. However, the $\tilde{O}(\sqrt{T})$ regret becomes less meaningful when the loss of the best expert $L^\star$ is small because the regret itself can be substantially greater than $L^\star$, rendering the learner's loss considerably higher than $L^\star$. As an illustration, suppose $L^\star = T^{1/4}$. An algorithm with regret $\sqrt{T}$ may incur a loss of $\sqrt{T}$, which is $T^{1/4}$ times larger than $L^\star$. In contrast, when $L^\star= T^{3/4}$ the total loss is only $(1+o(1))L^\star$. Thus, worst-case regret bounds are not suitable for OPE with small losses.

In the non-private setting, a natural remedy is to run the multiplicative weights method with a larger learning rate: increasing the learning rate from $\sqrt{\ln N/T}$ to $\sqrt{\ln N / L^\star}$ yields a regret bound of $O(\sqrt{L^\star\log N} + \log N)$~\citep{littlestone1994weighted,freund1997decision}. Unfortunately, such a strategy cannot be directly applied under DP. This is because a large learning rate amplifies the algorithm's sensitivity to individual data points, thereby violating privacy constraints. Indeed, all existing private OPE algorithms impose a strict upper bound on the learning rate to ensure DP, which prevents them from adapting to small $L^\star$. Consequently, it remains unclear whether an $\tilde{O}(\sqrt{L^\star})$ regret bound is achievable in the private setting.

\subsection{Our contribution}
\label{sec:results}
In this work, we assume that $\ell_1,\dots,\ell_T$ are chosen by an oblivious adversary. That is, they are fixed in advance and independent of the outputs of the learner. 

We propose the first differentially private algorithm that attains regret $\tilde{O}(\sqrt{L^\star})$ for OPE. This nearly matches the optimal regret without privacy constraints, suggesting that the cost of privacy remains small in the low-loss regime. The results are formally stated in the theorems below.

\begin{theorem}
\label{thm:main_pure}
    Let $\ell_1,\dots,\ell_T$ be chosen by an oblivious adversary. Given $\varepsilon\le 1$, there exists an $\varepsilon$-differentially private algorithm for OPE with expected regret
    \begin{equation*}
        O\left(\left(\sqrt{\log N\log T} + \log N \right)\cdot\sqrt{\frac{L^\star}{\varepsilon}}\right) + (1+\ln (1+L^\star)) \cdot\iota_1,
    \end{equation*}
    where $\iota_1 = O(\log^2N/\varepsilon + \log N\log T/\varepsilon)$.
\end{theorem}

\begin{theorem}
\label{thm:main_approx}
    Let $\ell_1,\dots,\ell_T$ be chosen by an oblivious adversary. Given $\varepsilon\le 1$ and $0<\delta \le 1/2$, there exists an $(\varepsilon,\delta)$-differentially private algorithm for OPE with expected regret
    \begin{equation*}
        O\left(\sqrt{\frac{L^\star\log N\log T}{\varepsilon}} + \log N\cdot \min\left(\frac{(L^\star\log (1/\delta))^{1/3}}{\varepsilon^{2/3}}, \sqrt{\frac{L^\star}{\varepsilon}}\right)\right) + (1+\ln (1+L^\star)) \cdot \iota_2,
    \end{equation*}
    where $\iota_2 = O\left(\log^{3/2}N\sqrt{\log(\min(1/\delta, N))}/\varepsilon + \log N \log T/\varepsilon\right)$.
\end{theorem}

\paragraph{Comparison with previous results.} When $L^\star = 0$, our results recover the bounds obtained by~\citet{asi2023near} for the realizable setting. When $L^\star$ grows, our algorithm yields a quadratic improvement over their $\tilde{O}(L^\star)$ regret bound. Compared to the $\tilde{O}(\sqrt{T})$ regret obtained by~\citet{asi2023private,asi2024private}, our results are significantly better when $L^\star$ is much smaller than $T$.

As shown by~\citet{asi2023private}, the above improvement cannot be achieved against an adaptive adversary under pure DP: any OPE algorithm must incur an $\Omega(T)$ regret even if $L^\star = 0$. We additionally prove that it remains impossible under approximate DP by providing an $\tilde{\Omega}(\sqrt{T})$ lower bound, which also holds when $L^\star = 0$.

\begin{theorem}
    Let $T$ be sufficiently large and $N\ge 2T$. Any adaptively $(\varepsilon,\delta)$-differentially private OPE algorithm with $\varepsilon\le 1$ and $\delta\le 1/T^3$ must incur an $\Omega(\min(T,\sqrt{T}/(\varepsilon\log T)))$ expected regret when $\ell_1,\dots,\ell_T$ are chosen by an adaptive adversary, even if it is promised that $L^\star=0$.
\end{theorem}

\subsection{Related work}

In the non-private setting, the OPE problem has been extensively studied. The multiplicative weights algorithm~\citep{littlestone1994weighted,vovk1990aggregating} was shown to achieve a data-dependent regret bound of $O(\sqrt{L^\star \log N} + \log N)$ with an appropriately tuned learning rate~\citep{freund1997decision}. There are also other methods with the same regret bound, such as follow-the-perturbed-leader~\citep{hannan1957approximation,kalai2005efficient,hutter2005adaptive}. It was formally proved by~\citet{cesa1997use} that such a bound is asymptotically optimal.

The private version of OPE was first considered by~\citet{dwork2010differential}. It is a folklore result that the multiplicative weights method preserves differential privacy, which follows from the exponential mechanism~\citep{mcsherry2007mechanism} and advanced composition~\citep{dwork2010boosting}. A series of subsequent works~\citep{smith2013nearly,jain2014near,agarwal2017price,asi2023near,asi2023private,asi2024private} further investigated the problem. In the oblivious adversarial setting considered in this paper,~\citet{asi2023private} proposed algorithms with regret $\tilde{O}(\sqrt{T})$ under pure DP and $O(\sqrt{T\log N}) + \tilde{O}(T^{1/3})$ under approximate DP, demonstrating that the cost of privacy is negligible. The bound for approximate DP was later refined by~\citet{asi2024private}, who further optimized the dependence on the privacy parameter $\varepsilon$. However, these results are unfavorable when $L^\star$ is much smaller than $T$.~\citet{asi2023near} studied the realizable setting, where $L^\star = 0$, and gave algorithms with $\tilde{O}(1)$ regret. Their algorithms can be extended to the agnostic setting where $L^\star > 0$. Nevertheless, the resulting regret bound is $\tilde{O}(L^\star)$, which is far from optimal.
\section{Preliminaries}

Online prediction from experts can be formulated as a $T$-round game played between a learner and an adversary. In this work, we assume an oblivious adversary who determines a sequence of loss functions $\ell_1,\dots,\ell_T$ in advance and releases them to the learner one by one, where every $\ell_t:[N]\to [0,1]$ assigns each of the $N$ experts a loss in $[0,1]$. In round $t$, the learner chooses an expert $i_t\in[N]$ before receiving $\ell_t$ and incurs loss $\ell_t(i_t)$. The performance of the learner is assessed by the expected regret, defined as
\begin{equation*}
    \mathbb{E}\left[\sum_{t=1}^{T}\ell_t(i_t) - \min_{i\in[N]}\sum_{t=1}^T\ell_t(i)\right].
\end{equation*}
For ease of notation, we write $L^\star$ to denote $\min_{i\in[N]}\sum_{t=1}^T\ell_t(i)$. Note that $L^\star$ is determined in advance because the adversary is oblivious. However, it is unknown to the learner until the end.

One common approach to solving OPE is the celebrated multiplicative weights method, which in each round $t$ samples $i_t$ such that
\begin{equation*}
    \Pr[i_t=i] = \frac{e^{-\eta\sum_{s=1}^{t-1}\ell_s(i)}}{\sum_{i'\in[N]}e^{-\eta\sum_{s=1}^{t-1}\ell_s(i')}},
\end{equation*}
where $\eta$ is the learning rate. The following lemma provides an upper bound on its expected regret.

\begin{lemma}[\citep{freund1997decision,cesa2006prediction}]
\label{lem:hedge}
    For $0<\eta \le 1$, we have
    \begin{equation*}
        \sum_{t=1}^T\sum_{i\in[N]}P_t(i)\ell_t(i) - L^\star =O\left(\frac{\log N}{\eta} + \eta L^\star\right),
    \end{equation*}
    where $P_t(i) = \frac{e^{-\eta\sum_{s=1}^{t-1}\ell_s(i)}}{\sum_{i'\in[N]}e^{-\eta\sum_{s=1}^{t-1}\ell_s(i')}}$.
\end{lemma}

We then introduce the notion of differential privacy in the context of OPE.
\begin{definition}[Differential Privacy~\citep{dwork2006calibrating,dwork2006our}]
    An algorithm $\A$ is $(\varepsilon,\delta)$-differentially private if, for any pair of sequences $S=(\ell_1,\dots,\ell_T)$ and $S'=(\ell_1',\dots,\ell_T')$ that differ in a single function, and for any subset $O$ of the output range, we have
    \begin{equation*}
        \Pr[\A(S)\in O]\le e^\varepsilon\Pr[\A(S')\in O] + \delta.
    \end{equation*}
    When $\delta =0$, we usually omit $\delta$ and say $\A$ is $\varepsilon$-differentially private. This case is referred to as pure DP, while the case where $\delta > 0$ is referred to as approximate DP.
\end{definition}



A key ingredient in our design is the $\abovethreshold$ mechanism, which receives a sequence of queries and privately identifies the first one that exceeds a prespecified threshold. 
\begin{theorem}[$\abovethreshold$~\citep{dwork2009complexity,dwork2014algorithmic}]
\label{thm:above}
    We say a query $q$ has sensitivity $1$ if $\lvert q(S) - q(S')\rvert\le 1$ for any pair of sequences $S$ and $S'$ that differ in a single function. Given a threshold $\tau$, there exists an $\varepsilon$-differentially private mechanism $\abovethreshold$ that at each round $t$ receives a sensitivity-$1$ query $q_t$ and outputs $\sigma_t\in\{\top,\bot\}$, continuing until it either produces the first $\top$ or has processed all queries. If there are at most $T$ queries, then with probability $1-\beta$, the $\abovethreshold$ mechanism is $\alpha$-accurate for 
    \begin{equation*}
        \alpha = \frac{8(\ln T+\ln(2/\beta))}{\varepsilon}.
    \end{equation*}
    Namely, if the mechanism halts after $k$ rounds, we have $q_t\ge \tau - \alpha$ for all $t\in[k]$ with $\sigma_t=\top$ and $q_t\le \tau + \alpha$ for all $t\in[k]$ with $\sigma_t=\bot$.
\end{theorem}

Every step of the multiplicative weights method is exactly the exponential mechanism with privacy parameter $\varepsilon=2\eta$. This property, combined with advanced composition, has been widely used in the literature to develop private OPE algorithms.
\begin{lemma}[Exponential Mechanism~\citep{mcsherry2007mechanism}]
\label{lem:exp}
    Let $P_t$ be defined as in Lemma~\ref{lem:hedge}. The process of sampling $i_t$ from $P_t$ is $2\eta$-differentially private.
\end{lemma}
\section{Algorithm and analysis}

In this section, we present our algorithm and analyze its privacy and regret guarantees. 

Our algorithm is built on the shrinking dartboard algorithm~\citep{geulen2010regret}, which was also leveraged by~\citet{asi2023private} to develop algorithms with $\tilde{O}(\sqrt{T})$ regret. The key insight is that the marginal distribution of the multiplicative weights method changes only slightly between consecutive rounds. Consequently, it suffices to update the output with only a small probability after each round. By reducing the number of updates, one can allocate a larger privacy budget for each update (i.e., use a larger learning rate $\eta$), thereby achieving smaller regret.

However, the shrinking dartboard algorithm is not inherently private because the probability of making an update depends on the underlying data (specifically, $\ell_t(i_t)$, i.e., the loss associated with the currently selected expert $i_t$).~\citet{asi2023private,asi2024private} addressed this by forcing the algorithm to perform an additional update at each round with a probability independent of the data. However, to satisfy even approximate DP, this probability must be at least $\tilde{\Omega}(T^{-1/3})$. By advanced composition, this imposes an $\tilde{O}(1/T^{1/3})$ upper bound on the learning rate $\eta$, which in turn yields an $\tilde{O}(T^{1/3})$ regret bound according to Lemma~\ref{lem:hedge}. This exceeds our desired $\tilde{O}(\sqrt{L^\star})$ regret whenever $L^\star < T^{2/3}$.

We address this by delaying the update until the currently selected expert has incurred a sufficiently large loss, and then making an update with a probability that depends on the cumulative loss accumulated during the delay. Our key consideration here is that, while such a probability is data-dependent, performing an update with this probability preserves privacy. Moreover, the delay before an update can be implemented using the $\abovethreshold$ mechanism, which is also differentially private. In addition, we impose a lower bound on the probability of updating and an upper bound on the number of updates, which are crucial for satisfying DP in the worst case. The overall algorithm is described in Algorithm~\ref{alg:main}. We formally prove its privacy guarantee in Section~\ref{sec:privacy} and regret bound in Section~\ref{sec:regret}.

\begin{algorithm}[!ht]
\label{alg:main}
\DontPrintSemicolon
    \KwInput{privacy parameters $\varepsilon_1$ and $\varepsilon_2$, learning rate $\eta$, threshold $\tau$, limit on updates $V$}
    $u\gets 1$, $v\gets 0$, and $t_0 \gets 1$ \;
    Sample $i_1$ uniformly from $[N]$ \;
    Initiate an instance of $\abovethreshold$ with privacy parameter $\varepsilon_1$ and threshold $\tau$ \;
    Observe $\ell_1$ and suffer loss $\ell_1(i_1)$ \;
    \For{$t=2,\dots,T$}{
        \uIf{$v < V$}{
            Define query $q_t\equiv\sum_{s=t_{u-1}}^{t-1}\ell_s(i_{t-1})$ \;
            Feed $q_t$ to $\abovethreshold$ and obtain outcome $\sigma_t$ \;
            \uIf{$\sigma_t=\top$}{
                $r_u \gets \max(1/(e^{\varepsilon_2} - 1), \sum_{s=t_{u-1}}^{t-1}\ell_s(i_{t-1}))$ \;
                Sample $Z_u$ from $\ber(e^{-\eta r_u})$ \;
                \uIf{$Z_u=1$}{
                    $i_t \gets i_{t-1}$\;
                }
                \Else{
                    Sample $i_t\gets i$ with probability $\frac{e^{-\eta\sum_{s=1}^{t-1}\ell_s(i)}}{\sum_{i'\in[N]}e^{-\eta\sum_{s=1}^{t-1}\ell_s(i')}}$ for $i\in[N]$ \;
                    $v\gets v+1$ \;
                }
                $t_u\gets t$ \;
                $u\gets u + 1$ \;
                Initiate an instance of $\abovethreshold$ with privacy parameter $\varepsilon_1$ and threshold $\tau$ \;
            }
            \Else{
                $i_t\gets i_{t-1}$ \;
            }
        }
        \Else{
            $i_t\gets i_{t-1}$ \;
        }
        Observe $\ell_t$ and suffer loss $\ell_t(i_t)$ \;
    }
\caption{Private online prediction from experts}
\end{algorithm}

\subsection{Privacy analysis}
\label{sec:privacy}
Let $\comp_\delta(k, \varepsilon_0)$ be the smallest $\varepsilon$ such that the composition of $k$ $\varepsilon_0$-differentially private algorithms is $(\varepsilon,\delta)$-differentially private. We prove the privacy guarantee of Algorithm~\ref{alg:main} in the following lemma.

\begin{lemma}
\label{lem:privacy}
    Algorithm~\ref{alg:main} is $(\varepsilon,\delta)$-differentially private for $\varepsilon = \varepsilon_1 + \max(\eta,\varepsilon_2) + \comp_\delta(V,2\eta)$.
\end{lemma}
\begin{proof}
    We analyze the augmented transcript consisting of the initial expert $i_1$, the outputs of all $\abovethreshold$ instances, the $Z_u$'s, and the experts generated at the (at most $V$) resampling steps. The learner's output sequence is a post-processing of this transcript.

    We view the algorithm as the concurrent composition of three types of interactive mechanisms: the collection of $\abovethreshold$ instances, the mechanism generating the $Z_u$'s, and the at most $V$ resampling mechanisms. Throughout the proof, when comparing two neighboring input sequences, we consider an arbitrary common prefix of the interaction transcript. Thus, all previously generated $\abovethreshold$ outputs, $Z_u$'s, resampled experts, and the resulting termination times are fixed and identical on the two executions.

    We first show that the collection of $\abovethreshold$ instances is $\varepsilon_1$-differentially private. Let the two neighboring sequences $S$ and $S'$ differ at round $\tilde{t}$. Along any fixed common interaction transcript, there is at most one $\abovethreshold$ instance whose query stream contains the differing loss. All earlier instances receive identical query streams and therefore have identical output distributions on $S$ and $S'$.

    Consider the potentially affected instance. Conditioned on the common transcript up to its launch, its starting time and selected expert are fixed and identical on the two executions. Hence, every query issued by this instance has sensitivity at most $1$. By the privacy guarantee of $\abovethreshold$, the transcript generated by this instance is $\varepsilon_1$-differentially private. If this instance terminates, then conditioned on a common transcript through its termination, every subsequent $\abovethreshold$ instance starts after round $\tilde{t}$. Consequently, all subsequent instances receive identical query streams on $S$ and $S'$. If the affected instance never terminates, there are no subsequent instances. Applying the chain rule to the conditional output distributions therefore shows that the collection of all $\abovethreshold$ instances, viewed as one interactive mechanism, is $\varepsilon_1$-differentially private.
    
    We next analyze the privacy loss from generating the $Z_u$'s. Let $Z_1,Z_2,\dots$ and $Z_1', Z_2',\dots$ denote the $Z_u$'s generated on neighboring sequences $S$ and $S'$, respectively. Similarly, let $r_u$'s and $r_u'$'s denote the $r_u$-values computed on $S$ and $S'$, respectively. As above, all probabilities involving $Z_u$’s are understood to be conditioned on the relevant fixed common transcript prefix.
    
    Along any common transcript path, the intervals $[t_{u-1},t_u-1]$ corresponding to generated $Z_u$'s are disjoint. Hence, there is at most one index $j$ such that $S$ and $S'$ differ at round $\tilde{t}\in[t_{j-1},t_j-1]$. If such $j$ does not exist, then $r_u = r_u'$ for every $Z_u$, which implies that the two $Z_u$-transcripts have identical conditional distributions.
    
    Otherwise, we have $r_u = r_u'$ for all $u< j$ and $\lvert r_j - r_j'\rvert\le 1$. Therefore,
    \begin{equation*}
        \frac{\Pr[Z_j = 1]}{\Pr[Z_j' = 1]} = \frac{e^{-\eta r_j}}{e^{-\eta r_j'}}\le e^\eta.
    \end{equation*}
    Moreover, since $\min(r_j,r_j')\ge 1/(e^{\varepsilon_2} - 1)$, we have
    \begin{equation*}
        \frac{\Pr[Z_j = 0]}{\Pr[Z_j' = 0]} = \frac{1-e^{-\eta r_j}}{1-e^{-\eta r_j'}} \le \frac{\min(r_j,r_j') + \lvert r_j - r_j'\rvert}{\min(r_j,r_j')}\le e^{\varepsilon_2},
    \end{equation*}
    where we use the fact that $\frac{1-e^{-y}}{1-e^{-x}}\le \frac{y}{x}$ for $0<x\le y$. For every later $u>j$, condition on an arbitrary common transcript prefix up to the generation of $Z_u$. The selected expert and the interval $[t_{u-1},t_u-1]$ are then fixed and identical on the two neighboring sequences. Since this interval occurs after the differing round, we have $r_u=r'_u$. Thus, $Z_u$ and $Z_u'$ have the same conditional distribution on the two executions. Therefore, by the chain rule, only $Z_j$ contributes privacy loss, and the generation of the $Z_u$'s is $\max(\eta,\varepsilon_2)$-differentially private.

    Finally, by Lemma~\ref{lem:exp}, each of the at most $V$ resampling steps is $2\eta$-differentially private conditioned on an arbitrary common transcript prefix. By the concurrent composition theorem~\citep{vadhan2021concurrent}, the adaptive interleaving of these mechanisms has the same composition guarantee as their non-interactive composition. Hence the augmented transcript, and therefore the learner's output sequence by post-processing, is $(\varepsilon,\delta)$-differentially private.
\end{proof}

\subsection{Regret analysis}
\label{sec:regret}
We proceed to analyze the regret. The following key lemma provides lower and upper bounds on the cumulative loss from round $t_{u-1}$ through round $t_u-1$.
\begin{lemma}
\label{lem:r_u}
    With probability $1-\beta_1$, we have
    \begin{equation*}
        \tau - \frac{8(\ln T + \ln(2T/\beta_1))}{\varepsilon_1}\le \sum_{s=t_{u-1}}^{t_u - 1}\ell_s(i_{t_{u}-1}) \le \max\left(\tau + \frac{8(\ln T + \ln(2T/\beta_1))}{\varepsilon_1}, 0\right) + 1
    \end{equation*}
    for all $u\ge 1$. The probability is taken solely over the randomness of the $\abovethreshold$ instances.
\end{lemma}
\begin{proof}
    By Theorem~\ref{thm:above} and the union bound, with probability $1-\beta_1$ we have $q_t\ge \tau - \alpha$ if $\sigma_t = \top$ and $q_t\le \tau + \alpha$ if $\sigma_t = \bot$ for all $t\in[T]$, where $\alpha = 8(\ln T + \ln(2T/\beta_1))/\varepsilon_1$. The lower bound follows directly from the fact that $\sigma_{t_u} = \top$. For the upper bound, the inequality trivially holds if $t_{u - 1} = t_u - 1$. When $t_{u - 1} < t_u - 1$, we have $\sigma_{t_u - 1} = \bot$ and $i_{t_u - 2} = i_{t_u - 1}$. Hence,
    \begin{equation*}
        \sum_{s=t_{u-1}}^{t_u - 1}\ell_s(i_{t_{u}-1})\le \sum_{s=t_{u-1}}^{t_u - 2}\ell_s(i_{t_{u}-2}) + 1\le \tau + \alpha + 1.
    \end{equation*}
\end{proof}

Next, we show that the number of updates is $\tilde{O}(\eta L^\star)$ with high probability. As a consequence, the algorithm works normally (i.e., it won't reach the update limit) if $V=\tilde{O}(\eta L^\star)$. By the composition theorems, this allows us to set $\eta = \tilde{\Theta}(1/\sqrt{L^\star})$, which minimizes the regret bound in Lemma~\ref{lem:hedge}.
\begin{lemma}
\label{lem:update}
    Let $\tau \ge \frac{8(\ln T + \ln(2T/\beta_1))}{\varepsilon_1} + \frac{1}{e^{\varepsilon_2} - 1}$. Conditioned on the event in Lemma~\ref{lem:r_u}, we have
    \begin{equation*}
        \Pr[v_T< 71(\ln N + \eta L^\star+1) + 280\ln(1/\beta_2)]\ge 1-\beta_2,
    \end{equation*}
    where $v_T$ is the value of $v$ when Algorithm~\ref{alg:main} terminates.
\end{lemma}
\begin{proof}
    Note that $v$ can increase by at most $1$ in each round. Therefore, the condition $v<V$ is always met given that $V\ge T$. In what follows, we will prove the result assuming $V \ge T$ because $v_T$ can only be smaller if we decrease $V$.
    
    Define a potential function
    \begin{equation*}
        \Phi_t = \sum_{i=1}^Ne^{-\eta\sum_{s=1}^{t-1}\ell_s(i)}.
    \end{equation*}
    We have $N=\Phi_1  \ge\cdots \ge \Phi_{T+1}\ge e^{-\eta L^\star}$. Let $s_j\in[T]$ be the largest index such that $\Phi_{s_j}\ge \Phi_1 / e^j$ and let $m$ be the smallest positive integer such that $s_m=T$. Then $m\le \ln(\Phi_1/\Phi_{T+1}) +1 \le \ln N+ \eta L^\star + 1$.

    Suppose an update happens at round $t\in[s_{j-1}+1,s_j]$ (set $s_0=0$ for simplicity). Let $E$ denote the event that another update happens at some round $t'\in[t +1,s_j]$. Let 
    \begin{equation*}
        S = \left\{i\in [N]: \sum_{s=t}^{s_j - 1}\ell_s(i)\le2/\eta\right\}.
    \end{equation*}
    By Lemma~\ref{lem:r_u}, we have $\sum_{s=t_{u-1}}^{t_u - 1}\ell_s(i_{t_{u}-1})\ge 1/(e^{\varepsilon_2} - 1)$ for any $u\ge 1$. This implies $r_u = \sum_{s=t_{u-1}}^{t_u - 1}\ell_s(i_{t_{u}-1})$. Thus, for any $i\in S$, we have
    \begin{equation*}
        \Pr[E\mid i_t =i]\le 1 - e^{-\eta\sum_{s=t}^{s_j - 1}\ell_s(i)}\le 1 - e^{-2}.
    \end{equation*}
    
    Note that 
    \begin{align*}
        \Phi_{s_j} &= \sum_{i\in S}e^{-\eta\sum_{s=1}^{s_j - 1}\ell_s(i)} + \sum_{i\notin S}e^{-\eta\sum_{s=1}^{s_j - 1}\ell_s(i)} \\
        &\le \sum_{i\in S}e^{-\eta\sum_{s=1}^{t - 1}\ell_s(i)} + e^{-2}\sum_{i\notin S}e^{-\eta\sum_{s=1}^{t - 1}\ell_s(i)} \\
        &= \Phi_t\Pr[i_t\in S] + e^{-2}\Phi_t\Pr[i_t\notin S] \\
        &=\Phi_{t}(1-\Pr[i_t\notin S]) + e^{-2}\Phi_{t}\Pr[i_t\notin S].
    \end{align*}
    Rearranging this yields
    \begin{equation*}
        \Pr[i_t\notin S] \le \frac{1 - \Phi_{s_j} / \Phi_{t}}{1 - e^{-2}}\le \frac{1-e^{-1}}{1-e^{-2}}.
    \end{equation*}
    Therefore,
    \begin{align*}
        \Pr[E]&\le \Pr[i_t\in S]\cdot \Pr[E\mid i_t\in S] + \Pr[i_t\notin S] \\
        & = (1-\Pr[i_t\notin S])\cdot (1-e^{-2}) + \Pr[i_t\notin S] \\
        &\le 1-e^{-2} + \frac{1-e^{-1}}{1-e^{-2}}\cdot e^{-2} \\
        &<0.97.
    \end{align*}
    
    As a consequence, the probability that $v_T \ge n + m$ can be bounded by the probability that the sum of $n$ Bernoulli random variables with mean $0.03$ is less than $m$. By the Chernoff bound, this is at most $\exp(-0.03n / 8)$ given that $0.03n\ge 2m$. Thus, it suffices to set $n = 70m + 280\ln(1/\beta_2)$.
\end{proof}

\begin{algorithm}[!ht]
\label{alg:aux}
\DontPrintSemicolon
    \KwInput{starting point $T_s$, ending point $T_e$, learning rate $\eta$, delay function $d: \mathbb{Z}^3\to\mathbb{N}$}
    $u\gets 1$ and $t_0 \gets T_s$ \;
    Sample $i_{T_s}\gets i$ with probability $\frac{e^{-\eta\sum_{s=1}^{T_s-1}\ell_s(i)}}{\sum_{j\in[N]}e^{-\eta\sum_{s=1}^{T_s-1}\ell_s(j)}}$ for $i\in[N]$ \;
    Observe $\ell_{T_s}$ and suffer loss $\ell_{T_s}(i_{T_s})$ \;
    \For{$t=T_s+1,\dots,T_e$}{
        \uIf{$t \ge d(t_{u-1}, i_{t_{u-1}}, u)$ or $t = T_e$}{
            $r_u \gets \sum_{s=t_{u-1}}^{t-1}\ell_s(i_{t-1})$ \;
            Sample $Z_u$ from $\ber(e^{-\eta r_u})$ \;
            \uIf{$Z_u=1$}{
                $i_t \gets i_{t-1}$\;
            }
            \Else{
                Sample $i_t\gets i$ with probability $\frac{e^{-\eta\sum_{s=1}^{t-1}\ell_s(i)}}{\sum_{j\in[N]}e^{-\eta\sum_{s=1}^{t-1}\ell_s(j)}}$ for $i\in[N]$ \;
            }
            $t_u\gets t$ \;
            $u\gets u + 1$ \;
        }
        \Else{
            $i_t\gets i_{t-1}$ \;
        }
        Observe $\ell_t$ and suffer loss $\ell_t(i_t)$ \;
    }
\caption{An auxiliary algorithm}
\end{algorithm}

The most challenging task in our proof is to compare the output distribution of Algorithm~\ref{alg:main} with that of the multiplicative weights method. To this end, we introduce Algorithm~\ref{alg:aux}, which makes the following modifications to Algorithm~\ref{alg:main}:
\begin{itemize}
    \item The limit $V$ on the number of updates and the lower bound on $r_u$ are removed.
    \item The $\abovethreshold$ mechanisms are replaced by a delay function $d$.
    \item The learning process can start from an arbitrary round $T_s$.
    \item The update procedure at round $T_e$ is forced.
\end{itemize}
We first prove that the output distributions of Algorithm~\ref{alg:aux} and the multiplicative weights method are the same at round $T_e$.
\begin{lemma}
\label{lem:aux}
    For any $T_s\in[T_e]$ and $d: \mathbb{Z}^3\to\mathbb{N}$, the output distribution of Algorithm~\ref{alg:aux} at round $T_e$ satisfies
    \begin{equation*}
        \Pr[i_{T_e} = i] = \frac{e^{-\eta\sum_{s=1}^{T_e-1}\ell_s(i)}}{\sum_{i'\in[N]}e^{-\eta\sum_{s=1}^{T_e-1}\ell_s(i')}}.
    \end{equation*}
\end{lemma}
\begin{proof}
    We prove by backward induction on $T_s$. When $T_s = T_e$, the conclusion holds since we sample $i_{T_e}$ exactly according to the desired distribution. Now consider $T_s=t'\in[T_e-1]$ and suppose that the result holds for all $T_s \ge t'+1$. Fix some $i_{t'}=k\in[N]$. Define $t_0 = t'$ and 
    \begin{equation*}
        t_j = \min(T_e,\max(t_{j-1}+1,d(t_{j-1}, i_{t'}, j))).
    \end{equation*}
    These are exactly the $t_u$'s in Algorithm~\ref{alg:aux} on the event $i_{t'}=\cdots=i_{T_e-1}=k$. Let $m$ be the first index such that $t_m=T_e$. Note that
    \begin{equation*}
        \Pr[Z_1=\cdots=Z_m=1 \mid i_{t'}=k] = e^{-\eta\sum_{s=t'}^{T_e-1}\ell_s(k)}.
    \end{equation*}
    Moreover, by the induction hypothesis, we have
    \begin{equation*}
        \Pr[i_{T_e}=i\mid Z_1=\cdots=Z_{j-1}=1, Z_j=0,i_{t'}=k] = \frac{e^{-\eta\sum_{s=1}^{T_e-1}\ell_s(i)}}{\sum_{i'\in[N]}e^{-\eta\sum_{s=1}^{T_e-1}\ell_s(i')}}
    \end{equation*}
    for any $j\in[m]$ since the learning process starting from $t_j$ can be seen as running Algorithm~\ref{alg:aux} with $d$ shifted by $j$ along the third dimension and $T_s=t_j$. For simplicity, write $w_t(i) = e^{-\eta\sum_{s=1}^{t-1}\ell_s(i)}$
    and
    $
        \Phi_t = \sum_{i\in[N]}w_t(i)$
    for $t\in[T_e]$. Note that $i_{T_e}\neq i_{t'}$ only if $Z_j=0$ for some $j$. Therefore, we have
    \begin{equation*}
        \Pr[i_{T_e}=i\mid i_{t'}=k] = \left(1 - e^{-\eta\sum_{s=t'}^{T_e-1}\ell_s(k)}\right)\cdot \frac{w_{T_e}(i)}{\Phi_{T_e}}
    \end{equation*}
    for $i\neq k$ and
    \begin{equation*}
        \Pr[i_{T_e}=k\mid i_{t'}=k] = e^{-\eta\sum_{s=t'}^{T_e-1}\ell_s(k)} + \left(1 - e^{-\eta\sum_{s=t'}^{T_e-1}\ell_s(k)}\right)\cdot \frac{w_{T_e}(k)}{\Phi_{T_e}}.
    \end{equation*}
    As a consequence, we have
    \begin{align*}
        \Pr[i_{T_e}=i]
        &=\sum_{k\in[N]}\Pr[i_{T_e}=i\mid i_{t'}=k]\cdot \Pr[i_{t'}=k] \\
        &=\sum_{k\in[N]}\left(1 - e^{-\eta\sum_{s=t'}^{T_e-1}\ell_s(k)}\right)\cdot \frac{w_{T_e}(i)}{\Phi_{T_e}}\cdot \frac{w_{t'}(k)}{\Phi_{t'}} + e^{-\eta\sum_{s=t'}^{T_e-1}\ell_s(i)}\cdot \frac{w_{t'}(i)}{\Phi_{t'}}\\
        &= \frac{w_{T_e}(i)}{\Phi_{T_e}}\cdot \frac{\Phi_{t'} - \Phi_{T_e}}{\Phi_{t'}} + \frac{w_{T_e}(i)}{\Phi_{t'}} \\
        &= \frac{w_{T_e}(i)}{\Phi_{T_e}}
    \end{align*}
    as desired.
\end{proof}
We leverage the fact that Algorithm~\ref{alg:aux} behaves similarly to Algorithm~\ref{alg:main} before round $T_e$ when $T_s=1$: Lemma~\ref{lem:r_u} and Lemma~\ref{lem:update} suggest that the first modification has little impact, while the second modification can be seen as fixing the random coins used by the $\abovethreshold$ mechanisms (see Appendix~\ref{sec:delay} for a more detailed explanation). By analyzing their behavior at round $T_e$ and applying the above lemma, we can argue that with high probability, the output distribution of Algorithm~\ref{alg:main} at each round is pointwise at most $e^{O(\eta(\tau+\log T/\varepsilon_1))}\approx 1+O(\eta\tau+\eta\log T/\varepsilon_1)$ times that of the multiplicative weights mechanism. This leads to the following regret bound.
\begin{lemma}
\label{lem:regret}
    Suppose $\tau\ge \frac{8(\ln T+2\ln(2T))}{\varepsilon_1} + \frac{1}{e^{\varepsilon_2} - 1}\ge 1$, $0<\eta\tau \le 1$, and $V\ge 71(\ln N + \eta L^\star+1) + 280\ln T$. The expected regret of Algorithm~\ref{alg:main} is at most
    \begin{equation*}
        O\left(\frac{\log N}{\eta} + \tau \eta L^\star\right).
    \end{equation*}
\end{lemma}
\begin{proof}
    Let $\beta_1 = 1/T$. Since the regret cannot be larger than $T$, it suffices to prove the desired conclusion conditioned on the event in Lemma~\ref{lem:r_u}. As argued in the proof of Lemma~\ref{lem:update}, for any $u\ge 1$ we have 
    \begin{equation*}
        r_u = \sum_{s=t_{u-1}}^{t_u - 1}\ell_s(i_{t_{u}-1}).
    \end{equation*}
    By Lemma~\ref{lem:r_u}, there exists $\Delta = \tau + O(\log T / \varepsilon_1) + 1=O(\tau)$ such that $r_u\le \Delta$ for all $u\ge 1$.

    We first prove the result under the condition that $V\ge T$. Fix a realization of the random coins used by the $\abovethreshold$ mechanisms and let $\hat{Q}_t(i)$ be the probability that Algorithm~\ref{alg:main} outputs $i_t=i$. Since $V\ge T$, there is no restriction on the number of updates. Hence, this realization induces a delay function $d$ so that Algorithm~\ref{alg:main} and Algorithm~\ref{alg:aux} (with this delay function $d$ and $T_s=1$) behave identically in the first $t-1$ rounds.

    Let $P_t$ be as defined in Lemma~\ref{lem:hedge} and consider any $i\in[N]$. By Lemma~\ref{lem:aux} with $1=T_s<T_e = t$, Algorithm~\ref{alg:aux} outputs $i_t=i$ with probability $P_t(i)$. Since it produces $i_{t-1}=i$ with probability $\hat{Q}_{t-1}(i)$, we have $P_t(i)\ge \hat{Q}_{t-1}(i)\cdot e^{-\eta\Delta}$. Now consider Algorithm~\ref{alg:main}. If it outputs $i_{t-1}\neq i$, then the probability that it outputs $i_t=i$ is at most $(1-e^{-\eta\Delta})P_t(i)$. As a consequence, we have
    \begin{align*}
        \hat{Q}_t(i) &\le \hat{Q}_{t-1}(i) + (1 - \hat{Q}_{t-1}(i))\cdot (1 - e^{-\eta\Delta})P_t(i) \\
        &\le e^{\eta\Delta}P_t(i)  + (1-e^{-\eta\Delta})P_t(i) \\
        &\le e^{2\eta\Delta}P_t(i),
    \end{align*}
    where the last inequality is due to the fact that $1+x-1/x\le x^2$ for $x\ge 1$. Moreover, we have $\hat{Q}_1(i) = P_1(i)$ since we sample $i_1$ uniformly at random. Thus, $\hat{Q}_t(i)\le e^{2\eta\Delta}P_t(i)$ for all $t\in[T]$. By Lemma~\ref{lem:hedge}, the expected regret of Algorithm~\ref{alg:main} with $V\ge T$ is
    \begin{align*}
        \sum_{t=1}^T\sum_{i\in[N]}\hat{Q}_t(i)\ell_t(i) - L^\star &\le e^{2\eta\Delta}\sum_{t=1}^T\sum_{i\in[N]}P_t(i)\ell_t(i) - L^\star \\
        &=e^{2\eta\Delta}\cdot O\left(\frac{\log N}{\eta} + \eta L^\star\right) + (e^{2\eta\Delta} - 1) L^\star \\
        &=O\left(\frac{\log N}{\eta} + \eta\tau L^\star\right)
    \end{align*}
    because $\eta\Delta = O(\eta\tau) = O(1)$.

    Let $Q_t(i)$ be the probability that Algorithm~\ref{alg:main} outputs $i_t=i$ with $V$ as in the lemma statement. Applying Lemma~\ref{lem:update} with $\beta_2 = 1/T$, we obtain $\Pr[v_T <  V] \ge 1 - 1/T$. We can decompose $Q_t(i)$ as $\Pr[v_T <  V]Q_t^0(i) + \Pr[v_T \ge  V]Q_t^1(i)$, where $Q_t^0(i)$ and $Q_t^1(i)$ are the probabilities of outputting $i_t=i$ conditioned on $v_T<V$ and $v_T\ge V$, respectively. We can also write $\hat{Q}_t(i) = \Pr[v_T <  V]\hat{Q}_t^0(i) + \Pr[v_T \ge  V]\hat{Q}_t^1(i)$ in the same way. For any $t\in[T]$, we have
    \begin{align*}
        \sum_{i\in[N]}\lvert Q_t(i) - \hat{Q}_t(i)\rvert &\le\sum_{i\in[N]}\Pr[v_T <  V]\left\lvert Q_t^0(i) - \hat{Q}_t^0(i)\right\rvert + \Pr[v_T \ge  V]\left\lvert Q_t^1(i)  -\hat{Q}_t^1(i)\right\rvert \\ 
        &\le 2\Pr[v_T \ge  V] \\
        &\le 2/T
    \end{align*}
    since $Q_t^0(i) = \hat{Q}_t^0(i)$. Consequently,
    \begin{equation*}
        \sum_{t=1}^T\sum_{i\in[N]}Q_t(i)\ell_t(i) - \sum_{t=1}^T\sum_{i\in[N]}\hat{Q}_t(i)\ell_t(i)  \le \sum_{t=1}^T\sum_{i\in[N]}\lvert Q_t(i) - \hat{Q}_t(i)\rvert \le 2.
    \end{equation*}
    Applying the triangle inequality leads to the desired result.
\end{proof}

Choosing $\eta = \tilde{\Theta}(1/\sqrt{L^\star})$ minimizes the regret bound in Lemma~\ref{lem:regret}. However, we cannot do this directly since $L^\star$ is unknown at the beginning. This issue can be circumvented by applying the standard doubling trick, which starts by assuming that $L^\star$ is small and doubles the estimated value once the estimate becomes inaccurate. We defer the detailed implementation to Appendix~\ref{sec:doubling}.

\section{Lower bound for adaptive adversaries}

In this section, we briefly discuss how we establish the lower bound for adaptive inputs. We utilize a lower bound in~\citep{asi2023private}, which states that every adaptively $(\varepsilon,\delta)$-differentially private OPE algorithm with $N\ge 2T$ experts must incur $\Omega(T)$ regret in the high-privacy regime where $\varepsilon = \tilde{O}(1/\sqrt{T})$, even if $L^\star = 0$. Now suppose $\A$ is an adaptively $(\varepsilon,\delta)$-differentially private OPE algorithm with $\varepsilon\le 1$ and regret at most $R_\A(T)$. We create an algorithm $\A_q$ that simulates $\A$, forwards $\A$'s outputs as its own, and feeds the loss function $\ber(q)\ell_t$ to $\A$. This can be seen as an extension of the canonical ``privacy amplification by subsampling'' technique to the online setting. We prove that $\A_q$ has privacy parameter $O(q\varepsilon)$ and regret at most $R_\A(T) / q$. Applying the lower bound in~\citep{asi2023private} to $\A_q$ with $q = \tilde{\Theta}(\min(1,1/(\varepsilon\sqrt{T})))$ yields $R_\A(T)/q = \Omega(T)$. Rearranging this gives $R_\A(T) = \Omega(qT) = \tilde{\Omega}(\min(T,\sqrt{T}/\varepsilon))$.

The main technical challenge lies in proving the privacy amplification result in the adaptive setting, which requires a more sophisticated analysis than in the oblivious setting. We defer the formal proof to Appendix~\ref{sec:lower}.
\section{Discussion and future work}
\label{sec:diss}

In this work, we propose the first differentially private algorithm for OPE that achieves a data-dependent regret bound of $\tilde{O}(\sqrt{L^\star})$, extending the previous $\tilde{O}(\sqrt{T})$ bound to the low-loss regime. Our result matches the non-private lower bound $\Omega(\sqrt{L^\star\log N} + \log N)$ up to poly-logarithmic factors for constant privacy parameters, suggesting that the cost of privacy is mild for OPE with small losses. Our algorithm works against any oblivious adversary. We also show that $\tilde{\Omega}(\sqrt{T})$ regret is necessary for adaptive inputs, demonstrating that such an improvement cannot be made against an adaptive adversary.

When $L^\star$ is large, our regret bounds are slightly worse than those in~\citep{asi2023private,asi2024private} in certain parameter ranges. For example, suppose $L^\star =\Theta(T)$. Under pure DP, ~\citet{asi2023private} gave an upper bound of $O(\sqrt{T}\log N / \varepsilon)$. Our bound in Theorem~\ref{thm:main_pure} is larger than theirs by a factor of $\sqrt{\varepsilon\log T/\log N}$ when this factor is greater than $1$ (when it is less than $1$, our bound is better). Under approximate DP,~\citet{asi2024private} proved an upper bound of $O(\sqrt{T\log N} + T^{1/3}\log N\log(T/\delta)/\varepsilon^{2/3})$. The first term in our bound in Theorem~\ref{thm:main_approx} is larger than that in their bound by a factor of $O(\sqrt{\log T/\varepsilon})$. An interesting future direction is to investigate whether our bounds can be further improved to match these results.

\begin{ack}
The research was supported in part by an NSFC grant 62432008, RGC RIF grant R6021-20, an RGC TRS grant T43-513/23N-2, RGC CRF grants C7004-22G, C1029-22G and C6015-23G, NSFC/RGC grant CRS\_HKUST601/24 and RGC GRF grants 16207423, 16203824 and 16216725.
\end{ack}

\bibliographystyle{plainnat}
\bibliography{ref}


\appendix

\section{Algorithm without prior knowledge of $L^\star$}
\label{sec:doubling}
In this section, we show how to construct an algorithm without prior knowledge of $L^\star$. We first prove regret bounds when some estimate $L'\ge L^\star$ is given.
\begin{theorem}
    Let $\varepsilon\le 1$ and $L'>0$. There exists an $\varepsilon$-differentially private algorithm for OPE with expected regret
    \begin{equation*}
        O\left(\left(\log N + \sqrt{\log N\log T}\right)\cdot\sqrt{\frac{L'}{\varepsilon}}\right) + \iota_1
    \end{equation*}
    whenever $L'\ge L^\star$, where $\iota_1 = O(\log^2N/\varepsilon + \log N\log T/\varepsilon)$. 
\end{theorem}
\begin{proof}
    We run Algorithm~\ref{alg:main} with $\varepsilon_1=\varepsilon_2= \varepsilon/3$, $\tau = 8(\ln T+2\ln(2T))/\varepsilon_1 + 1/(e^{\varepsilon_2} - 1)$, $\eta = \min(\varepsilon/(1278\ln N + 1296),\varepsilon / (5040\ln T), \sqrt{\varepsilon/(1278L')}, \sqrt{\varepsilon\ln N / (L'\ln T)})$, and $V = \lceil71(\ln N + \eta L'+1) + 280\ln T\rceil$. One can verify that $\eta\le \varepsilon/3$, $2\eta V\le \varepsilon/3$, and $\eta\tau\le 1$. By Lemma~\ref{lem:privacy} and basic composition, the overall algorithm is $\varepsilon$-differentially private. By Lemma~\ref{lem:regret} and the fact that $L^\star\le L'$, the expected regret is bounded by
    \begin{equation*}
        O\left(\frac{\log N}{\eta} + \frac{\log T }{\varepsilon}\cdot\ \eta L'\right).
    \end{equation*}
    Plugging in the value of $\eta$ yields the desired result.
\end{proof}

\begin{theorem}
    Let $\varepsilon\le 1$, $0<\delta \le 1/2$, and $L'>0$. There exists an $(\varepsilon,\delta)$-differentially private algorithm for OPE with expected regret 
    \begin{equation*}
        O\left(\sqrt{\frac{L'\log N\log T}{\varepsilon}} + \log N\cdot \min\left(\frac{(L'\log (1/\delta))^{1/3}}{\varepsilon^{2/3}}, \sqrt{\frac{L'}{\varepsilon}}\right)\right) + \iota_2
    \end{equation*}
    whenever $L'\ge L^\star$, where $\iota_2 = O\left(\log^{3/2}N\sqrt{\log(\min(1/\delta, N))}/\varepsilon + \log N\log T/\varepsilon\right)$.
\end{theorem}
\begin{proof}
    We run Algorithm~\ref{alg:main} with $\varepsilon_1=\varepsilon_2= \varepsilon/3$, $\tau = 8(\ln T+2\ln(2T))/\varepsilon_1 + 1/(e^{\varepsilon_2} - 1)$, $\eta = \min(\eta_1,\eta_2,\eta_3,\eta_4,1/\tau)$, and $V = \lceil71(\ln N + \eta L' + 1) + 280\ln T\rceil$, where
    \begin{itemize}
        \item $\eta_1 = \max(\varepsilon/(18\cdot \lceil71\ln N + 72\rceil), \varepsilon/(36\sqrt{2\cdot \lceil71\ln N+72\rceil\ln(3/\delta)}))$.
        \item $\eta_2 = \max(\varepsilon/(18\cdot\lceil280\ln T\rceil), \varepsilon/(36\sqrt{2\cdot\lceil280\ln T\rceil\ln(3/\delta)}))$.
        \item $\eta_3 = \max(\sqrt{\varepsilon/(1278L')}, \varepsilon^{2/3}/(184032L'\ln(3/\delta))^{1/3})$.
        \item $\eta_4 = \sqrt{\varepsilon\ln N/(L'\ln T)}$.
    \end{itemize}
    Note that $\lceil71(\ln N + \eta L'+1) + 280\ln T\rceil\le \lceil71 \ln N +72\rceil + \lfloor71\eta L'\rfloor + \lceil 280\ln T\rceil$. By basic and advanced composition, we have
    \begin{align*}
        {}&\comp_\delta(V,2\eta) \\
        \le& \comp_{\delta/3}(\lceil71\ln N+72\rceil,2\eta) + \comp_{\delta/3}(\lfloor71\eta L'\rfloor, 2\eta)+\comp_{\delta/3}(\lceil280\ln T\rceil,2\eta) \\
        \le& \varepsilon / 9 + \varepsilon / 9 + \varepsilon / 9 \\
        =&\varepsilon/3.
    \end{align*}
    Applying Lemma~\ref{lem:privacy} proves the privacy guarantee. Since $\eta\tau\le 1$, by Lemma~\ref{lem:regret} and the fact that $L^\star\le L'$, the expected regret is bounded by
    \begin{equation*}
        O\left(\frac{\log N}{\eta} + \frac{\log T }{\varepsilon}\cdot \eta L'\right).
    \end{equation*}
    Plugging in the value of $\eta$ yields the desired result.
\end{proof}

\begin{algorithm}[!ht]
\label{alg:no_prior}
\DontPrintSemicolon
    $k \gets 1$, $b_1\gets 1$, and $\alpha \gets \frac{16(\ln T+2\ln(2T))}{\varepsilon}$ \;
    Initiate an instance of $\abovethreshold$ with privacy parameter $\varepsilon/2$ and threshold $4\alpha$\;
    Initiate an instance of $\alg$ with $L'=6\alpha$ \;
    \For{$t=1,\dots,T$}{
        Obtain $i_t$ from $\alg$ \;
        Observe $\ell_t$ and suffer loss $\ell_t(i_t)$ \;
        Define query $q_t \equiv \min_{i\in[N]}\sum_{s=b_k}^t\ell_s(i)$ \;
        Feed $q_t$ to $\abovethreshold$ and obtain outcome $\sigma_t$ \;
        \If{$\sigma_t=\top$}{
            $k\gets k + 1$ \;
            $b_k\gets t + 1$ \;
            Terminate the existing instance of $\alg$ and initiate a new one with $L'=3\cdot 2^k\alpha$ \;
            Initiate an instance of $\abovethreshold$ with privacy parameter $\varepsilon/2$ and threshold $(3\cdot 2^k - 2)\alpha$\;
        }
    }
\caption{Algorithm without prior knowledge of $L^\star$}
\end{algorithm}

The overall algorithm works as follows. We start with $L' = 6\alpha$ and employ an $\abovethreshold$ mechanism to test if the current cumulative loss of the best expert exceeds $L'$. Once this happens, we double the value of $L'$ and initiate new instances of both the learning algorithm and the $\abovethreshold$ mechanism. This procedure is formalized in Algorithm~\ref{alg:no_prior} and the regret guarantee is established in the lemma below. Applying this framework to the two preceding theorems yields our main results in Section~\ref{sec:results}.

\begin{lemma}
    If $\alg$ is an $(\varepsilon/2,\delta)$-differentially private OPE algorithm for every $L'>0$ and has expected regret at most $R(L')$ whenever $L'\ge L^\star$, then Algorithm~\ref{alg:no_prior} is $(\varepsilon,\delta)$-differentially private and has expected regret at most
    \begin{equation*}
        1 + \sum_{k=1}^{K} R(3\cdot 2^k\alpha),
    \end{equation*}
    where $\alpha = \frac{16(\ln T+2\ln(2T))}{\varepsilon}$ and $K = \lfloor\log_2(\max(4L^\star/(3\alpha), 2))\rfloor$.
\end{lemma}
\begin{proof}
    We first prove the privacy guarantee. Consider the augmented transcript consisting of the termination times of all $\abovethreshold$ instances together with the outputs of all instances of $\mathrm{Alg}$.

    We first claim that the sequence of $\abovethreshold$ termination times is $\varepsilon/2$-differentially private. Consider two neighboring loss sequences that differ at round $\tilde{t}$, and condition on an arbitrary common prefix of the termination-time transcript. There is at most one $\abovethreshold$ instance whose queries contain the differing loss. All earlier instances receive identical query streams. For the affected instance, every query has sensitivity at most $1$, so its termination time is $\varepsilon/2$-differentially private. Conditioned on a common termination time for this instance, every subsequent instance starts strictly after round $\tilde{t}$, and therefore receives identical query streams on the two neighboring sequences. By the chain rule, the entire sequence of termination times is $\varepsilon/2$-differentially private.

    Now fix any sequence of termination times. This also fixes a partition of the loss sequence into disjoint subsequences. The instances of $\alg$ are run independently on these disjoint subsequences, and at most one of them receives the differing loss. By parallel composition, their joint output is $(\varepsilon/2,\delta)$-differentially private. Applying basic composition to the termination-time transcript and the outputs of $\alg$, followed by post-processing, shows that Algorithm 3 is $(\varepsilon,\delta)$-differentially private.
    
    We then prove utility. By Theorem~\ref{thm:above} and the union bound, with probability $1-1/T$ all applications of $\abovethreshold$ are $\alpha$-accurate. We first prove a regret bound conditioned on this event. Suppose $\sigma_t = \top$ at round $t$. Before increasing the value of $k$, we have
    \begin{equation*}
        (3\cdot 2^k-3)\alpha \le q_t \le q_{t-1} + 1\le (3\cdot 2^k -1)\alpha + 1 < 3\cdot 2^k\alpha.
    \end{equation*}
    By our assumption, the algorithm $\alg$ with $L'=3\cdot 2^k\alpha$ incurs an expected regret of at most $R(3\cdot 2^k\alpha)$ during $[b_k, t] = [b_k, b_{k+1}-1]$. Moreover, we have
    \begin{equation*}
        L^\star = \min_{i\in [N]}\sum_{t=1}^T\ell_t(i) \ge q_t\ge (3\cdot 2^k - 3)\alpha \ge 3\cdot 2^{k-1}\alpha.
    \end{equation*}
    This implies that
    \begin{equation*}
        3\cdot 2^{k+1} \alpha \le 4L^\star.
    \end{equation*}
    Hence, whenever a new bucket $k+1$ is initiated, we have $k+1\le K$. Since the algorithm starts with $k=1$, the total number of buckets is therefore at most $K$. For the final bucket, if no $\top$ is returned, the accuracy of $\abovethreshold$ implies $q_T\le (3\cdot2^k-1)\alpha<3\cdot2^k\alpha$, so the same regret bound applies. Summing the regret over all buckets gives $\sum_{k=1}^{K}R(3\cdot 2^k\alpha)$.

    Since the regret can be at most $T$, the unconditional expected regret is at most
    \begin{equation*}
        (1-1/T)\sum_{k=1}^{K}R(3\cdot 2^k\alpha) + 1/T\cdot T \le 1 + \sum_{k=1}^{K} R(3\cdot 2^k\alpha).
    \end{equation*}
\end{proof}
\section{Construction of the delay function}
\label{sec:delay}
In this section, we show how to construct the delay function when the randomness of the $\abovethreshold$ mechanisms is fixed. Consider an $\abovethreshold$ mechanism that is initiated at round $s$. It is implemented as in Algorithm~\ref{alg:abovethreshold}.

\begin{algorithm}[!ht]
\label{alg:abovethreshold}
\DontPrintSemicolon
    \KwInput{privacy parameter $\varepsilon$, threshold $\tau$}
    $\xi \gets \lap(2/\varepsilon)$ \;
    \For{$t=s+1,\dots,T$}{
        $\nu_t\gets \lap(4/\varepsilon)$ \;
        Receive $q_t$ \;
        \uIf{$q_t + \nu_t\ge\tau + \xi$}{
            $\sigma_t\gets \top$ \;
            \textbf{Halt}. \;
        }
        \Else{
            $\sigma_t\gets \bot$ \;
        }
    }
\caption{$\abovethreshold$}
\end{algorithm}

Note that in each round $t\in[T]$ we always draw a random variable $\nu_t$. Hence, we can fix the values of $\nu_1,\dots,\nu_T$ at the beginning. Additionally, we draw $\xi_1,\dots,\xi_T$, where $\xi_u$ denotes the variable $\xi$ used by the $u$-th $\abovethreshold$ mechanism. 

Now suppose that in Algorithm~\ref{alg:main} we initiate the $u$-th $\abovethreshold$ mechanism at round $t_{u-1}$. Then it will return $\sigma_t = \top$ for the first $t$ satisfying
\begin{equation*}
    \sum_{s=t_{u-1}}^{t-1}\ell_s(i_{t_{u-1}}) + \nu_t \ge \tau+\xi_u.
\end{equation*}
As a consequence, we can define $d(t_{u-1},i_{t_{u-1}},u)$ as this $t$ if such a $t$ exists, and as $T+1$ otherwise. Doing this for every tuple $(t_{u-1},i_{t_{u-1}},u)\in[T]\times[N]\times[T]$ constructs the delay function $d$.
\section{Lower bound for adaptive adversaries}
\label{sec:lower}
Before presenting the proof, we first formally define differential privacy against an adaptive adversary. Let $\A$ be an online learner, $\D$ be an adversary, and $b\in\{0, 1\}$ be a hidden bit that is unknown to both $\A$ and $\D$. Define the privacy game played between $\A$ and $\D$ as follows. At each round $t\in[T]$, the adversary $\D$ generates two loss functions $\ell_t^0$ and $\ell_t^1$ and the algorithm produces an output $i_t$, both based on the interaction history up to round $t-1$. After these choices are made, $\A$ receives $\ell_t^b$ while $\D$ observes $i_t$. During the entire game, the adversary $\D$ ensures that $\ell_t^0=\ell_t^1$ for all $t\in[T]\setminus\{\tilde{t}\}$, where $\tilde{t}$ is a special round that is chosen adaptively by $\D$ (i.e., the choice of $\tilde{t}$ can depend on the interaction history before round $\tilde{t}$). Let $\mathsf{View}_{\A,\D}(b)$ denote the view of $\D$ in this privacy game, including the internal randomness of $\D$ and $(i_1,\dots,i_T)$ produced by $\A$. Differential privacy against adaptive adversaries is defined as follows.

\begin{definition}[Differential Privacy against Adaptive Adversaries~\citep{asi2023private}]
    An OPE algorithm $\A$ is adaptively $(\varepsilon,\delta)$-differentially private if for any adversary $\D$ and any set $O$ of views, we have
    \begin{equation*}
        \Pr[\mathsf{View}_{\A,\D}(0)\in O]\le e^\varepsilon\Pr[\mathsf{View}_{\A,\D}(1)\in O] + \delta.
    \end{equation*}
\end{definition}

Our proof is based on the following lower bound from~\citep{asi2023private} that holds in the high-privacy regime.

\begin{lemma}[High-Privacy Lower Bound~\citep{asi2023private}]
\label{lem:high}
    Let $T$ be sufficiently large and $N\ge 2T$. For any adaptively $(\varepsilon,\delta)$-differentially private OPE algorithm with $\varepsilon\le1/(\sqrt{T}\ln T)$ and $\delta \le T^{-3}$, there exists an adaptive adversary generating $\ell_1,\dots,\ell_T$ such that $L^\star = 0$ and 
    \begin{equation*}
        \mathbb{E}\left[\sum_{t=1}^{T}\ell_t(i_t)\right] = \Omega(T).
    \end{equation*}
\end{lemma}
\begin{remark}
    The above lemma can be obtained by setting $\varepsilon\le1/(\sqrt{T}\ln T)$ and $\delta\le T^{-3}$ in Theorem 9 of~\citep{asi2023private}. While the theorem statement does not explicitly state that $L^\star = 0$, the proof actually constructs a realizable input sequence.
\end{remark}

Our idea can be seen as extending the privacy amplification technique to the online setting.

\begin{definition}[Poisson Sampling]
\label{def:poi}
    Given an OPE algorithm $\A$ and $0< q \le1$, define a new algorithm $\A_q$ as follows. It runs a copy of $\A$. At each round $t$, it first draws $Z_t\sim\ber(q)$. Then, it forwards that copy's output $i_t$ as its own. After observing $\ell_t$, it feeds $\tilde{\ell}_t= Z_t\ell_t$ to the copy.
\end{definition}

We first establish the privacy guarantee of $\A_q$. In the proof, we write $P\approx_{(\varepsilon,\delta)}Q$ to denote that distributions $P$ and $Q$ are $(\varepsilon,\delta)$-close, i.e., they satisfy the inequalities defining DP:
\begin{equation*}
    P(O)\le e^\varepsilon Q(O)+\delta \qquad\text{and}\qquad Q(O)\le e^\varepsilon P(O)+\delta
\end{equation*}
for any event $O$.
\begin{lemma}
\label{lem:ampl}
   Let $\A$ and $\A_q$ be defined as in Definition~\ref{def:poi}. If $\A$ is adaptively $(\varepsilon,\delta)$-differentially private, then $\A_q$ is adaptively $(\ln(1+q(e^\varepsilon-1)),q\delta)$-differentially private.
\end{lemma}
\begin{proof}
    Let $\D$ be any adversary that interacts with $\A_q$. Write
    \begin{equation*}
        V_b = \mathsf{View}_{\A_q, \D}(b)
    \end{equation*}
    for $b\in\{0,1\}$. Let $P_b$ be the distribution of $V_b$. Introduce the following three auxiliary view distributions:
    \begin{itemize}
        \item $P^-$: the distribution of $\D$'s view in an auxiliary privacy game where $Z_{\tilde t}$ is set deterministically to $0$, while all other sampling coins are drawn as usual.
        \item $P^+_0$: the distribution of $\D$'s view in an auxiliary privacy game with $b=0$ where $Z_{\tilde t}$ is set deterministically to $1$, while all other sampling coins are drawn as usual.
        \item $P^+_1$: the distribution of $\D$'s view in an auxiliary privacy game with $b=1$ where $Z_{\tilde t}$ is set deterministically to $1$, while all other sampling coins are drawn as usual.
    \end{itemize}
    Since $\tilde{t}$ is chosen based on the interaction history before round $\tilde{t}$, while $Z_{\tilde{t}}$ is sampled independently at round $\tilde{t}$, we have
    \begin{equation*}
        P_b = (1 - q) P^- + q P_b^+,
    \end{equation*}
    where the first term is due to the fact that $\A$ receives $\tilde{\ell}_{\tilde{t}} = Z_{\tilde{t}}\ell_{\tilde{t}}^b = 0$ whenever $Z_{\tilde t} = 0$, regardless of the value of $b$.

    Now we derive privacy relations among $P^-$, $P^+_0$, and $P^+_1$. We first consider $P^+_0$ and $P^+_1$. Construct an adversary $\D^+$ that simulates $\D$ and at each round $t\neq \tilde{t}$ draws $Z_t\sim\ber(q)$ and generates
    \begin{equation*}
        (\tilde{\ell}_t^0,\tilde{\ell}_t^1) = (Z_t\ell_t^0,Z_t\ell_t^1).
    \end{equation*}
    At round $\tilde{t}$, it generates
    \begin{equation*}
        (\tilde{\ell}_{\tilde{t}}^0,\tilde{\ell}_{\tilde{t}}^1) = (\ell_{\tilde{t}}^0,\ell_{\tilde{t}}^1).
    \end{equation*}
    It is easy to check that $\D^+$ is a valid adversary in the privacy game with $\A$: it only generates different loss functions at round $\tilde{t}$. For $b\in\{0,1\}$, the resulting interaction is exactly the branch defining $P^+_b$. By the privacy guarantee of $\A$ and the fact that the view of $\D$ can be seen as a post-processing of $\D^+$'s view, we have $P^+_0\approx_{(\varepsilon,\delta)}P^+_1$.

    Next we compare $P^+_0$ with $P^-$. Construct an adversary $\D^-$ that performs the same simulation as $\D^+$ at round $t\neq \tilde{t}$. At round $\tilde{t}$, it generates
    \begin{equation*}
        (\tilde{\ell}_{\tilde{t}}^0,\tilde{\ell}_{\tilde{t}}^1) = (\ell_{\tilde{t}}^0,0).
    \end{equation*}
    For $b=0$, the interaction is again the branch defining $P^+_0$. For $b=1$, the interaction becomes the branch defining $P^-$. Therefore, we have $P^+_0\approx_{(\varepsilon,\delta)}P^-$. We can similarly prove that $P^+_1\approx_{(\varepsilon,\delta)}P^-$.

    We now prove the final conclusion. For any measurable set $O$, write $x = P^+_0(O)$, $y = P^+_1(O)$, and $z= P^-(O)$. By the pairwise $(\varepsilon,\delta)$-closeness of $P^+_0$, $P^+_1$, and $P^-$, we have $x\le e^\varepsilon y + \delta$ and $x\le e^\varepsilon z + \delta$. Hence,
    \begin{align*}
        P_0(O) &=(1-q)z + qx \\
        &\le (1-q)z + qe^\varepsilon\min(y,z) +q\delta \\
        &\le (1+q(e^\varepsilon - 1))((1-q)z + qy) + q\delta \\
        &= (1+q(e^\varepsilon - 1))P_1(O) + q\delta,
    \end{align*}
    where the second inequality follows because
    \begin{align*}
        {}&(1+q(e^\varepsilon - 1))((1-q)z + qy) - ((1-q)z + qe^\varepsilon\min(y,z)) \\
        \ge& q(1-q)(e^\varepsilon - 1)(z - \min(y,z)) + q(1+q(e^\varepsilon - 1))(y - \min(y,z)) \\
        \ge&0.
    \end{align*}

    By symmetry, $P_1(O)\le (1+q(e^\varepsilon - 1))P_0(O) + q\delta$. This proves the desired result.
\end{proof}

\begin{lemma}
\label{lem:scale}
    Let $\A$ and $\A_q$ be defined as in Definition~\ref{def:poi}. If $\A$ has expected cumulative loss at most $R_{\A}(T)$ for any adaptively generated realizable input sequence, then $\A_q$ has expected cumulative loss at most $R_{\A}(T)/q$ for any adaptively generated realizable input sequence.
\end{lemma}
\begin{proof}
    Let $\ell_1,\cdots,\ell_T$ be a realizable sequence generated by an adaptive adversary. There exists $i^\star\in [N]$ such that $\ell_t(i^\star) = 0$ for all $t\in[T]$. By definition, $\tilde{\ell}_{t}(i^\star) =Z_t\ell_t(i^\star) = 0$. This implies that the sequence $\tilde{\ell}_1,\cdots,\tilde{\ell}_T$ is also realizable. Therefore,
    \begin{equation*}
        \mathbb{E}\left[\sum_{t=1}^TZ_t\ell_t(i_t)\right] =\mathbb{E}\left[\sum_{t=1}^T\tilde{\ell}_t(i_t)\right]\le R_{\A}(T).
    \end{equation*}
    
    On the other hand, since $Z_t$ is independent of $i_t$ and $\ell_t$, we have
    \begin{equation}
        \mathbb{E}\left[Z_t\ell_t(i_t)\right] = \mathbb{E}\left[Z_t\right]\mathbb{E}\left[\ell_t(i_t)\right] = q\mathbb{E}\left[\ell_t(i_t)\right].
    \end{equation}
    Summing over $t$ gives
    \begin{equation*}
        q\mathbb{E}\left[\sum_{t=1}^T\ell_t(i_t)\right] = \mathbb{E}\left[\sum_{t=1}^TZ_t\ell_t(i_t)\right]\le R_A(T).
    \end{equation*}
    Rearranging the inequality yields the desired conclusion.
\end{proof}

\begin{theorem}
    Let $T$ be sufficiently large and $N\ge 2T$. For any adaptively $(\varepsilon,\delta)$-differentially private OPE algorithm with $\varepsilon\le1$ and $\delta \le T^{-3}$, there exists an adaptive adversary generating $\ell_1,\dots,\ell_T$ such that $L^\star = 0$ and 
    \begin{equation*}
        \mathbb{E}\left[\sum_{t=1}^{T}\ell_t(i_t)\right] = \Omega\left(\min\left\{T, \frac{\sqrt{T}}{\varepsilon\log T}\right\}\right).
    \end{equation*}
\end{theorem}
\begin{proof}
    Suppose $\A$ is adaptively $(\varepsilon,\delta)$-DP and has expected cumulative loss at most $R_{\A}(T)$ on realizable input sequences. Consider $\A_q$ in Definition~\ref{def:poi} with 
    \begin{equation*}
        q = \min\left\{1,\frac{1}{2\varepsilon\sqrt{T}\ln T}\right\}.
    \end{equation*}
    Lemma~\ref{lem:ampl} establishes that $\A_q$ is adaptively $(\ln(1+q(e^\varepsilon-1)),q\delta)$-differentially private. Since $\varepsilon \le 1$, we have
    \begin{equation*}
        \ln(1+q(e^\varepsilon-1))\le 2q\varepsilon \le 1/(\sqrt{T}\ln T).
    \end{equation*}
    Moreover,
    \begin{equation*}
        q\delta\le \delta\le T^{-3}.
    \end{equation*}
    Therefore, $\A_q$ is adaptively $(1/(\sqrt{T}\ln T), T^{-3})$-differentially private. By Lemma~\ref{lem:high}, $\A_q$ must incur an expected cumulative loss of $\Omega(T)$. Lemma~\ref{lem:scale} then yields 
    \begin{equation*}
        R_\A(T) = \Omega(qT) = \Omega\left(\min\left\{T, \frac{\sqrt{T}}{\varepsilon\log T}\right\}\right).
    \end{equation*}
\end{proof}



\end{document}